%% file: main.tex
\documentclass[10pt,twocolumn,letterpaper]{article}

\usepackage{wacv}              

\usepackage{graphicx}

\usepackage{amsthm, amsmath, amssymb}
\usepackage{booktabs}
\usepackage{color}
\usepackage[noend]{algpseudocode}
\usepackage[linesnumbered,vlined,ruled,commentsnumbered]{algorithm2e}

\usepackage[accsupp]{axessibility}  
\usepackage{multirow} 
\usepackage{subcaption}
\usepackage{caption}
\usepackage{booktabs}
\usepackage{float}
\usepackage{mathtools}

\newtheorem{definition}{Definition}
\newtheorem{theorem}{Theorem}
\newtheorem{lemma}[theorem]{Lemma}
\newcommand{\norm}[1]{\lVert#1\rVert}

\DeclareMathOperator*{\argmin}{arg\,min}

\usepackage[pagebackref,breaklinks,colorlinks]{hyperref}

\usepackage[capitalize]{cleveref}
\crefname{section}{Sec.}{Secs.}
\Crefname{section}{Section}{Sections}
\Crefname{table}{Table}{Tables}
\crefname{table}{Tab.}{Tabs.}

\def\wacvPaperID{145} 
\def\confName{WACV}
\def\confYear{2025}

\begin{document}

\title{Defending Against Repetitive Backdoor Attacks on Semi-supervised Learning through Lens of Rate-Distortion-Perception Trade-off}

\author{
Cheng-Yi Lee$^{*}$\\
{\normalsize Academia Sinica}\\
\and
Ching-Chia Kao$^{*}$\\
{\normalsize National Taiwan University}\\
\and
Cheng-Han Yeh\\
{\normalsize Academia Sinica}\\
\and
Chun-Shien Lu\\
{\normalsize Academia Sinica}\\
\and
Chia-Mu Yu\\
{\normalsize National Yang Ming Chiao Tung University}\\
\and
Chu-Song Chen\\
{\normalsize National Taiwan University}\\
}

\maketitle
\def\thefootnote{*}\footnotetext{Equal contribution}

\begin{abstract}
   Semi-supervised learning (SSL) has achieved remarkable performance with a small fraction of labeled data by leveraging vast amounts of unlabeled data from the Internet. However, this large pool of untrusted data is extremely vulnerable to data poisoning, leading to potential backdoor attacks. Current backdoor defenses are not yet effective against such a vulnerability in SSL. In this study, we propose a novel method, Unlabeled Data Purification (UPure), to disrupt the association between trigger patterns and target classes by introducing perturbations in the frequency domain. By leveraging the Rate-Distortion-Perception (RDP) trade-off, we further identify the frequency band, where the perturbations are added, and justify this selection. Notably, UPure purifies poisoned unlabeled data without the need of extra clean labeled data. Extensive experiments on four benchmark datasets and five SSL algorithms demonstrate that UPure effectively reduces the attack success rate from $99.78\%$ to $0\%$ while maintaining model accuracy. Code is available here: \url{https://github.com/chengyi-chris/UPure}.

\end{abstract}

\section{Introduction}
\label{sec:intro}
Deep learning models have shown superiority in computer vision tasks. However, supervised learning relies heavily on human-labeled training data. The process of labeling, which is costly~\cite{culotta2005reducing} and error-prone~\cite{li2017learning}, renders traditional machine learning methods impractically expensive for real-world applications.

Semi-supervised learning (SSL)~\cite{berthelot2019remixmatch,sohn2020fixmatch,zhang2021flexmatch} reduces the need for labeled data by leveraging a large amount of unlabeled data. By combining limited labeled data and extensive unlabeled data, SSL exhibits performance that is close to or better than that achieved by supervised learning. For example, due to its effective use of unlabeled data, the SSL algorithm, FixMatch~\cite{sohn2020fixmatch}, achieves $92\%$ classification accuracy on CIFAR10 with only \textit{100} labeled data and \textit{50k} unlabeled data. In general, SSL algorithms involve two loss functions: a supervised loss function on labeled training data (\eg, cross-entropy~\cite{murphy2012machine}) and an unsupervised loss function on unlabeled training data (\eg, cross-entropy over pseudo-labels~\cite{lee2013pseudo}). The design of the unsupervised loss contributes to the difference between SSL algorithms.

\subsection{Unlabeled data poisoning}
Several studies~\cite{carlini2023poisoning,saha2022backdoor,yan2021deep} have explored the potential threat posed by uncurated data in collection pipelines. Backdoor attacks~\cite{gu2019badnets,li2022backdoor}, one of the vulnerabilities, allow an adversary to manipulate a fraction of the training data by injecting the trigger (\ie, a particular pattern). This causes malicious behavior in the model during training with poisoned data. The backdoor model behaves normally with benign samples, but generates malicious and targeted predictions when the backdoor is activated. However, previous work \cite{feng2022unlabeled,feng2023unlabeled,yan2021dehib,yan2021deep} assumes that an adversary has access to SSL's labeled training data, which may not be appropriate for a realistic threat model. Without knowledge of labeled data in SSL, a recent study~\cite{shejwalkar2023perils} injects poisoned samples with a low poisoning rate (\ie, $0.2\%$) to achieve $97\%$ success rate on only unlabeled data. Therefore, the goal of our study is to defend against data poisoning attacks on unlabeled data in the context of SSL. 

\subsection{Our Backdoor Defense on SSL} We investigate backdoor defense for SSL in real-world applications, where users may acquire a substantial amount of unlabeled data from untrusted third-party sources, along with a fraction of their own labeled data. Specifically, we present a simple yet effective backdoor defense, \textbf{Unlabeled data Purification (UPure)}, to cleanse the untrusted unlabeled data before training a model. We first transform the unlabeled training images into the frequency domain. UPure does not assume the knowledge about the model structure~\cite{liu2018trojaning}, training data~\cite{biggio2012poisoning,munoz2017towards,turner2019label}, and learning algorithms, nor does it rely on backdoor detectors~\cite{chen2018detecting,chou2020sentinet,gao2019strip,hayase2021spectre,liu2023detecting,qi2023towards,tran2018spectral,zeng2021rethinking} to identify poisoned samples before training. Inspired by the Rate-Distortion-Perception (RDP) trade-offs, we introduce three different strategies to purify the high-frequency components of unlabeled data. Finally, the modified images are transformed back into the pixel domain. Through this design, we not only preserve the fidelity of images but also effectively mitigate backdoor attacks. Furthermore, based on the recently developed RDP trade-offs~\cite{blau2019rethinking}, we provide a theoretical analysis to explain the rationale behind the reason for invalidating triggers and choosing perturbation areas.

We systematically conduct experiments on four benchmark image datasets with five commonly used SSL algorithms and compare them to five existing defenses that are agnostic to learning algorithms. The results show that UPure improves the robustness of SSL and outperforms the existing defenses,  with only a slight decline in benign accuracy. 

The main contributions in this paper are three-fold. 
\begin{enumerate}
    \item Under a realistic scenario, where an adversary cannot access labeled training data of SSL, we present a simple yet effective backdoor defense, UPure, against unlabeled data poisoning attacks on SSL. UPure purifies poisoned unlabeled data before training, without needing extra clean labeled data.
    \item Based on the RDP trade-off, we reveal from analysis the importance of selecting the appropriate perturbation region in the frequency domain. In addition, we also provide a theoretical justification to explain why triggers are ineffective after applying UPure.
 
    
    \item In terms of both Benign Accuracy (BA) and Attack Success Rate (ASR), UPure achieves superior performance against backdoor attacks using five famous SSL algorithms.

\end{enumerate}

\section{Related Work}
\label{sec:related}
\subsection{Semi-Supervised Learning}

SSL reduces the dependence of the model on labeled data by incorporating both labeled dataset $D_{\ell}$ and unlabeled dataset $D_u$. Despite different distributions between $D_{\ell}$ and $D_u$, the number $|D_u|$ of unlabeled data is significantly larger than the number $|D_{\ell}|$ of labeled data (\textit{i.e.}, $|D_u|\gg|D_{\ell}|$). In general, SSL loss function can be formulated as a combination of a supervised loss $\mathcal{L}_{\ell}$ on $D_{\ell}$ and an unsupervised loss $\mathcal{L}_u$ on $D_u$; {\em i.e.}, $\mathcal{L}_{ss\ell} = \mathcal{L}_{\ell} + \lambda\mathcal{L}_u$. Typically, $\mathcal{L}_{\ell}$ refers to the standard cross-entropy loss, favored for its high performance. However, the choice of $\mathcal{L}_u$ differs among various SSL algorithms.

In recent years, SSL~\cite{berthelot2019mixmatch,sohn2020fixmatch,xie2020unsupervised} has made significant progress, particularly with the introduction of MixMatch~\cite{berthelot2019mixmatch}. MixMatch combines diverse data augmentation methods with existing SSL algorithms, resulting in notable performance enhancements. This study concentrates on SSL algorithms that employ pseudo-labeling~\cite{lee2013pseudo} and consistency regularization~\cite{laine2016temporal,sajjadi2016regularization} techniques. On the one hand, pseudo-labeling assigns labels to unlabeled samples based on their highest class probabilities. On the other hand, consistency regularization aims to enforce consistency in predictions made by the model on unlabeled data under different perturbations or augmentations. The idea is to encourage the model to learn more robust and generalizable representations of the unlabeled samples, avoiding producing inconsistent predictions for the same unlabeled sample. Please refer to the supplementary about the details of SSL algorithms considered in this work.

\subsection{Backdoor Attack}
Backdoor attacks~\cite{gu2019badnets,saha2020hidden,li2022backdoor} establish a malicious relationship between the trigger and target label during training and activate backdoors at inference. Backdoor attacks garner significant attention, especially within the realm of SSL~\cite{carlini2021poisoning,feng2022unlabeled,shejwalkar2023perils,yan2021dehib,yan2021deep}. However, \cite{shejwalkar2023perils} highlighted that previous approaches \cite{carlini2021poisoning,feng2022unlabeled,yan2021dehib,yan2021deep} are impractical for realistic scenarios, where an adversary can only manipulate unlabeled data without access to labeled training data of SSL. For instance, without knowledge of labeled data, \cite{carlini2021poisoning,yan2021dehib} achieved attack success rates of only up to 33\% and 37.5\%, respectively. The backdoor triggers in~\cite{feng2022unlabeled,yan2021deep} require the semantic information on target labeled image distribution to train a SSL model. In addition, \cite{wang2023manifold} postulates that the adversary can use the labeled training data in SSL, which does not align with realistic scenarios.


To address the above issues, \cite{shejwalkar2023perils} emphasizes that backdoor attacks in SSL should be agnostic to the distribution of training-labeled images. It is imperative to select poisoning data from the targeted class and ensure that backdoor trigger patterns\footnote{Repetitive patterns cover the entire image; when zooming in any part, they exhibit similar patterns, as outlined in~\cite{shejwalkar2023perils}.} repetitively cover the entire poisoned sample to prevent compromise by strong data augmentations, such as cutouts~\cite{devries2017improved}\footnote{Cutout involves selectively removing a specific portion or object from an image as a data augmentation technique.}, commonly used in modern SSL algorithms. In this study, our primary objective is to propose a defensive approach against the backdoor attacks outlined in~\cite{shejwalkar2023perils} on SSL.

\subsection{Backdoor Defense}

Prevailing backdoor defense methods fall into two categories: (1) \textit{Pre-processing defense} \cite{chen2022effective,huang2021backdoor,li2021anti,zhang2023backdoor} train a clean model based on their defense principle from a given poisoned dataset. (2)\textit{Post-processing defense} \cite{li2020neural,liu2018fine,zeng2021adversarial,qiu2021deepsweep} attempts to purify a given model with a small fraction of benign data, which may be unrealistic in the real world. In addition, many methods either incur computational costs for identifying poisoned samples during training or require (re)training from scratch to purify a backdoored model, often rendering them time-consuming. In contrast, our approach relies solely on rapid and straightforward data manipulation, serving as a pre-processing defense that enhances robustness during training, and effectively mitigates backdoors to ensure practicality and convenience for users.

\subsection{Rate-Distortion-Perception Trade-off}
Shannon's groundbreaking research on rate-distortion theory \cite{shannon1959coding} explores the essential balance between the rate required to represent data and the distortion that arises when the data is reconstructed. Recent studies~\cite{blau2018perception,blau2019rethinking,zhang2021universal} have integrated perceptual quality into the rate-distortion theory~\cite{shannon1959coding}, illustrating a three-way trade-off among rate, distortion, and perception.  Our study aims to disrupt the association between trigger patterns and their related target classes by selective perturbations while minimizing distortion and maintaining perceptual quality. To quantify the effect of perturbations on image quality and further explicate our approach, we delve into the rationale behind our method through the lens of RDP trade-offs, as elaborated in~\cref{Sec: RDP}. 

\section{Our Proposed Backdoor Defense}

\subsection{Problem Formulation}

\noindent \textbf{SSL.} In general, SSL requires very little labeled data $D_{\ell} = \{(\mathbf{x}_i, y_i)\}^n_{i=1}$ and a huge amount of unlabeled data $D_{u} = \{(\mathbf{x}_i)\}^m_{i=n+1}$ to form the training dataset $D = D_{\ell} \cup D_{u}$, where $|D_{\ell}| \ll |D_{u}|$. Let $\mathbf{x}_i$ be an image sample from a distribution $\mathcal{X} \subset \mathbb{R}^{H \times W \times C}$ (\ie, a dataset also refers to an image manifold in~\cref{fig1-a}) and let its ground truth label be denoted as $y_i \in \mathcal{Y} = \{1, \dots, \mathcal{K}\}$, where $\mathcal{K}$ denotes the number of classes in the dataset.
With a customized SSL algorithm, the classifier $f$ outputs the class with the highest confidence as the prediction. 

\noindent \textbf{Threat model.}
Consistent with previous backdoor attacks \cite{shejwalkar2022back,shejwalkar2023perils}, the adversary is assumed only to manipulate unlabeled data of the SSL pipeline, but does not access model structure, training loss, and the labeled data. As outlined in~\cite{shejwalkar2023perils}, the repetitive trigger pattern in SSL, which does not require labeled data, is practical in real-world scenarios. In this study, we consider the repetitive trigger for our backdoor defense.

\noindent \textbf{Defender's Goal.} A common setting for the backdoor defenses in SSL \cite{li2020neural,liu2018fine,zeng2021adversarial,qiu2021deepsweep} is that a defender can have full control over the training process. However, since the defender faces a massive amount of unlabeled data, this indicates the potential presence of poisoned samples in the dataset. The defender's goal is to purify the unlabeled training data such that the model trained on the purified data is not infected by the backdoor. We do not assume that the defender has additional clean datasets, which would be more practical for real-world scenarios.

\subsection{Unlabeled data Purification: UPure}\label{sec:3.2} 

\begin{figure}[!htbp]
  \centering
  \begin{subfigure}[b]{0.48\linewidth}
    \centering
    \includegraphics[width=\linewidth]{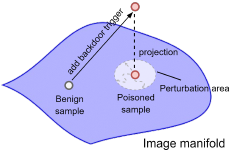}
    \caption{Conceptualization of backdoor attack}
    \label{fig1-a}
  \end{subfigure}
  \hspace{0.02\linewidth}
  \begin{subfigure}[b]{0.48\linewidth}
    \centering
    \includegraphics[width=\linewidth]{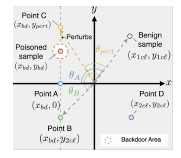}
    \caption{Local Mapping to 2D Euclidean space}
    \label{fig1-b}
  \end{subfigure}

  \caption{Illustration of UPure: (a) A benign sample can be turned into a malicious one by adding a backdoor trigger that moves it outside the image manifold and then projects onto the image manifold by clipping it to become a poisoned sample. (b) A local mapping that maps the benign and poisoned samples to 2D space. Points A, B, and C represent three strategies in the frequency purification step in UPure (\cref{sec:3.2}), which renders the poisoned sample ineffective.} 
  \label{fig:perturb_idea}

\end{figure}

As depicted in \cref{fig1-a}, poisoned samples exhibit distinct representations on the image manifold from the benign ones. One can see that even if we identify both the exact backdoor trigger and sample, we cannot reconstruct the benign image due to the projection operator (\ie, clipping, which is a multi-to-one function). Our goal is to steer poisoned samples away from the backdoor's effective zone by inducing a random walk with some directional and perturbation constraints. We consider a frequency space illustration (shown in~\cref{fig1-b}), where $x$-axis represents low-frequency components and $y$-axis represents high-frequency components. Here, an image is composed of a low-frequency and a high-frequency components. The subscript $c\ell$ denotes ``clean,'' $bd$ denotes ``backdoored,'' and $pert$ denotes ``perturbation''. We can see from such a space that if only the $y$-axis is changed, the angle $\theta_{pert}$ between the perturbed/purified poisoned sample (\ie, point C) and its benign counterpart will be smaller than $\theta_A$ (between point A and benign sample) and $\theta_B$ (between point B and benign sample), where $\theta_A$ is obtained by setting $y_{bd}$ to 0 and $\theta_B$ is obtained by replacing $y_{bd}$ with $y_{2c\ell}$ of another benign training sample. These angles (\ie, $\theta_A$, $\theta_B$, and $\theta_{pert}$) indicate the quality of the resulted images, compared to the benign sample; the smaller the angle, the better the fidelity. The aforementioned low-frequency and high-frequency components are also illustrated as the blue and red areas when it comes to the spectrum of an image, respectively, in~\cref{fig:proposed}.


\begin{figure}[!ht]
  \centering
  \centerline{\includegraphics[width=0.9\linewidth]{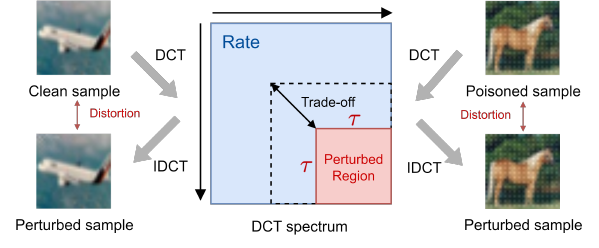}}
  \caption{Connection between RDP and our backdoor defense method on SSL. Specifically, UPure perturbs the high-frequency component in DCT spectrum to preserve the perceptual quality. By the RDP trade-off, we can derive the size of the perturbation zone (\ie distortion) to eliminate the effect of repetitive backdoor.}

\label{fig:proposed}
\end{figure}

Based on the above observations, we propose that UPure purify the training images before the training process. UPure consists of three steps: transformation to the frequency domain, frequency purification, and transformation to the pixel domain. The first and third steps are accomplished by discrete cosine transform (DCT) \cite{DCT} and Inverse DCT (IDCT). For the second step, we study three strategies, including ``\textit{Turn to zero},'' ``\textit{Replace from others},'' and ``\textit{Add perturbation},'' (points A, B, and C in~\cref{fig:perturb_idea}). Different strategies purify the high-frequency spectrum in different ways, as depicted in~\cref{fig:proposed}. Given the constraints of distortion and perception quality, this allows us to preserve an appropriate area (\ie, rate) in the DCT spectrum. However, perturbing low-frequency components in the DCT spectrum removes semantic features, prompting us to select high-frequency components as perturbation regions. Compared to diffusion-based backdoor defense~\cite{shi2023black}, UPure serves as a form of data augmentation and offers a more efficient approach for processing large amounts of unlabeled data.


Nevertheless, an important yet unexplored problem remains: \textbf{How to characterize an appropriate perturbation region (\ie, the bottom-right red region of size $\tau\times\tau$ in~\cref{fig:proposed})? Specifically, what is the theoretical maximal size of this region?} To answer this question, the RDP trade-offs \cite{blau2019rethinking,zhang2021universal} are examined in the next section.

\subsection{Theoretical Analysis}\label{Sec: RDP} 


Our theoretical analysis can be divided into two stages. First, we introduce the RDP trade-offs~\cite{blau2019rethinking} and demonstrate that UPure fits within the RDP paradigm, where there exists a theoretical lower bound of rate as an optimal threshold. Second, we derive the trigger failure probability under a single or repetitive trigger pattern, which means a lower bound in defense based on this optimal threshold.



\subsubsection{Rate-Distortion-Perception trade-offs} 
The RDP function \cite{blau2019rethinking} aims to study the relation between the input $X$ and the output $\hat{X}$ of an encoder-decoder pair, and is a mapping defined by a conditional distribution $p_{\hat{X}|X}$.

\begin{definition}
The Rate-Distortion-Perception function \cite{blau2019rethinking} is defined as
\vspace{-0.1in}
\begin{align}\label{eq:RDP}
R(D,P) = \min_{p_{\hat{X}|X}} \,\, I(X,\hat{X}) 
 \nonumber \\
\,\, \text{\emph{s.t.}} \,\, \mathbb{E}[\Delta(X,\hat{X})] \le D,\,\, d(p_X,p_{\hat{X}}) \le P,
\end{align}
where $I$ denotes mutual information \cite{cover2012elements}, $\mathbb{E}[\Delta(X,\hat{X})]$ denotes the expected distortion of decoded signals w.r.t the joint distribution $p_{X,\hat{X}}=p_{\hat{X}|X}p_X$, $\Delta:\mathcal{X} \times \hat{\mathcal{X}} \rightarrow \mathbb{R}^+$ is any distortion measure such that $\Delta(x,\hat{x})=0$ if and only if $x=\hat{x}$, and $d(\cdot,\cdot)$ is some divergence between distributions. Here, we assume that $d(p,q)\ge0$, and $d(p,q)=0$ if and only if $p=q$.
\end{definition}


The RDP function works for an encoder-decoder architecture; in our case, we consider DCT as the encoder and IDCT as the decoder. In the encoder-decoder architecture, a quantizer converts a continuous range of values into discrete values, which dominates the fidelity of the decoded images. Let $\mathcal{Q}_C$, $x$, and $g(x)$ be the set of quantization centers,  the input, and the output of the encoder before quantization, respectively. In the following, we study deterministic quantization (DQ) and noisy quantization (NQ).

\noindent \textbf{Deterministic Quantization (DQ).} The defender computes
\begin{align}\label{eq:dq}
    z = \argmin_{c \in \mathcal{Q}_C} \norm{g(x) - c}
\end{align}
and delivers $z$ to the user for decoding. 
Specifically, $\mathcal{Q}_C$ in \cref{eq:dq} are the spectrums with zero high frequency component for ``\textit{Turn to zero}'' (points A in~\cref{fig1-b}). For 
``\textit{Replace from other}'', let $\Omega = \{\omega : \text{spectrums of benign samples}\}$. The user first computes $z' = z \oplus \omega$, which means replacing the high-frequency component of $z$ by $\omega$ and then passing $z'$ to the decoder to get the desired output (points B in~\cref{fig1-b}). 
On the other hand, NQ has only one type of quantizer, corresponding to ``\textit{Add perturbation}'' in UPure.

\noindent \textbf{Noisy Quantization (NQ).} The defender also computes~\cref{eq:dq}
and sends $z$ to the user for decoding. 
Given $\epsilon > 0$, the user computes $z'' = z \oplus \eta$, where $\eta \sim \mathcal{N}(0,\sigma^2 I)$ represents a perturbation with $\|\eta\| > \epsilon$, serving as the constraint to force the solution $z$ out of the backdoor region (points C in~\cref{fig1-b}). 

Theoretically, for a fixed distortion $D$ resulting from DQ or NQ, we simulate several $R(D,P)$ curves with different perceptual quality $P$ in the supplementary. We claim that if we desire to obtain a lower $P$ (better quality), we have to devise a better quantizer. To validate our claim, we empirically evaluate two different quantizers, namely, ``\textit{Turn to zero}'' and ``\textit{Replace from other},'' from DQ and ``\textit{Add perturbation}'' from NQ in ~\cref{expts}.

\subsubsection{Destroy Trigger Patterns} 

Consider two cases of backdoor attacks using single~\cite{yan2021dehib,yan2021deep} or repetitive~\cite{shejwalkar2023perils} triggers in SSL. The first situation is the probability that a single trigger cannot succeed under the cutout~\cite{devries2017improved} operation, as listed in Lemma~\ref{lem:1}. The second situation discusses the probability of repeated trigger failure through UPure (by perturbing on high-frequency components), as shown in Lemma~\ref{lem:2}. The detailed proofs are provided in the supplementary. In summary, we have the following~\cref{thm:pf} regarding the lower bound of the failure probability between these triggers.




\begin{lemma}\label{lem:1}
    Consider an image of size $\mathsf{H} \times \mathsf{W}$ with a single trigger pattern bounded by a rectangle $\mathsf{H}_t \times \mathsf{W}_t$ and the cutout~\cite{devries2017improved} region is $\mathsf{H}_c \times \mathsf{W}_c$, where $\mathsf{H} > \mathsf{H}_c \ge \mathsf{H}_t$ and $\mathsf{W} > \mathsf{W}_c \ge \mathsf{W}_t$. 
    Let $h$ and $w$ be the cover distance between the cutout region and the trigger along the horizontal (vertical) direction. We have the failure probability $p_f^{single}$ of a single trigger pattern
    \begin{align}
        p_f^{single} \ge \frac{\sum_{w=0}^{\left\lfloor \mathsf{W}_t - \frac{\alpha}{\mathsf{H}_t} \right\rfloor}\left[\left\lfloor \Phi(w) \right\rfloor + 1\right]}{(\mathsf{H}-\mathsf{H}_c+1)(\mathsf{W}-\mathsf{W}_c+1)},
    \end{align}
where $\Phi(w) = \mathsf{H}_t - \frac{\alpha}{(\mathsf{W}_t - w)} = h$ and $1 \leq \alpha \leq \mathsf{H}_t \times \mathsf{W}_t$. 
\end{lemma}
Following Lemma~\ref{lem:1}, a single trigger pattern will be ineffective if the overlapping area between the trigger and the cutout region surpasses a certain size $\alpha$. Furthermore, we consider the repetitive trigger pattern with high-frequency characteristics and its failure probability bounds as follows.

\begin{lemma}\label{lem:2}
    Consider a $\mathsf{N}$-pixel image, which has $\mathsf{N}$ coefficients in the DCT-based frequency domain. The $k$-th coefficient is changed with probability $q_k$, where $\mathsf{M}+1\leq k\leq \mathsf{N}-1$ with $\mathsf{M}$ as a threshold and $k$ follows the zig-zag order. When $\beta$ coefficients are changed, we have the failure probability $p_f^{repet}$ of a repetitive trigger pattern
    \begin{align}\label{eq:lem2}
        p_f^{repet} \ge \sum_{\mathsf{L}=\beta}^{\mathsf{N}-\mathsf{M}-1} \binom{\mathsf{N}-\mathsf{M}-1}{\mathsf{L}} \prod_{k \in \mathsf{S}}q_k\prod_{k \in \mathsf{S}^\mathsf{C}}(1 - q_k),
    \end{align}
    where $\mathsf{S}$ is a set of frequency indexes in which the corresponding coefficients are changed, and $|\mathsf{S}| = \mathsf{L}$ and $|\mathsf{S}^\mathsf{C}| = (\mathsf{N}-\mathsf{M}-1)-\mathsf{L}$ that are index sets chosen from the universal index set $\mathsf{U}$\footnote{The universal index set $\mathsf{U} = \{\mathsf{M}+1, \mathsf{M}+2, \cdots, \mathsf{N}-1\} = \mathsf{S} \cup \mathsf{S}^\mathsf{C}$, indicating that the rate to be preserved is $\mathsf{M}$ and the other coefficients can be discarded.}.
\end{lemma}

As mentioned in~\cite{shejwalkar2023perils}, cutout operations~\cite{devries2017improved} also destroy low frequency in SSL training. Here, we focus on high frequency (a repetitive trigger pattern) in Lemma~\ref{lem:2}, which posits that the probability of failure increases with changes in certain high-frequency coefficients. The impact of variables $\mathsf{M}$, $\beta$, and $q_k$ on Lemma~\ref{lem:2} is illustrated in~\cref{sec:discussion}. In addition, $|\mathsf{S}|$ also denotes the perturbation region located in the bottom-right of DCT spectrum in~\cref{fig:proposed}.


 

\setcounter{theorem}{0}
\begin{theorem}\normalfont (The success probability of defense).\label{thm:pf}
    Suppose that cutout~\cite{devries2017improved} and UPure are applied during SSL training. In this scenario, we have a success probability $p_f$, which represents the failure probability of the trigger, to purify poisoned samples.
    \begin{align}\label{eq:pb}
    \scalebox{0.84}{$
        p_f \ge \left[\sum\limits_{\mathsf{L}=\beta}^{\mathsf{N}-\mathsf{M}-1} \binom{\mathsf{N}-\mathsf{M}-1}{\mathsf{L}} 
        \prod\limits_{k \in \mathsf{S}}q_k\prod\limits_{k \in \mathsf{S}^\mathsf{C}}(1 - q_k)\right]\left[\frac{\sum_{w=0}^{\lfloor \mathsf{W}_t - \frac{\alpha}{\mathsf{H}_t} \rfloor}\left[\lfloor \Phi(w) \rfloor + 1\right]}{(\mathsf{H}-\mathsf{H}_c+1)(\mathsf{W}-\mathsf{W}_c+1)}\right]$}
    \end{align}
\end{theorem}
\begin{proof}
    According to Lemma~\ref{lem:1}, a single trigger pattern is destroyed with cutout operation~\cite{devries2017improved}. Moreover, according to Lemma~\ref{lem:2}, a repetitive trigger pattern is invalid for UPure. Therefore, we can directly obtain a success probability of defense $p_f \ge p_f^{single} \cdot p_f^{repet}$.
\end{proof}
\cref{thm:pf} facilitates the calculation of a threshold value $\beta$ using the RDP function. This threshold determines the number of coefficients $|\mathsf{S}|$ that the bottom-right region of size in DCT spectrum should be perturbed. Consequently, it allows us to establish a lower bound for the success rate in defense for a given $\alpha$.


\section{Experiments}\label{expts}
\subsection{Experimental Setup}

\noindent \textbf{Datasets and model.} We evaluate our experiments on four benchmark datasets (CIFAR10~\cite{krizhevsky2009learning}, SVHN~\cite{netzer2011reading}, STL10~\cite{coates2011analysis}, and CIFAR100~\cite{krizhevsky2009learning}). For training, we utilize WideResNet~\cite{zagoruyko2016wide} as model architecture and follow the original hyper-parameters specified in USB benchmark~\cite{wang2022usb}, as aligned with~\cite{shejwalkar2023perils}. It is worth noting that STL10~\cite{coates2011analysis} (a scaled-down variant of ImageNet1K~\cite{deng2009imagenet}) using the SSL algorithm is still cost-expensive for computation, \ie, about $48$ hours to complete $20,000$ iterations. We refer to the supplementary for more implementation details. 

\input{tables/tab-main_result}

\noindent \textbf{Attack settings.} In this paper, we use backdoors from~\cite{shejwalkar2023perils} rather than those used in previous attacks~\cite{feng2022unlabeled,yan2021dehib}. This choice was made because, without knowledge of labeled training data, \cite{feng2022unlabeled} achieved relatively low attack success rates of up to 33\% and 37.5\%, respectively. In addition, this work~\cite{shejwalkar2023perils} points out that the perturbation-based trigger~\cite{yan2021dehib} is not effective against SSL, achieving only a 1\% ASR on CIFAR10. Therefore, we do not consider this case in our experiments. For the backdoor attack, we fix a poisoning rate $\gamma = 0.2\%$. In addition, we aim to compromise the entire test dataset with a backdoor trigger, rather than just a single sample in~\cite{carlini2021poisoning}.


\noindent \textbf{Defense settings.} We compare UPure with five backdoor defenses, \ie, Fine-tuning (FT), Fine-pruning (FP)~\cite{liu2018fine}, Neural Attention Distillation (NAD)~\cite{li2020neural}, Backdoor Adversarial Unlearning (I-BAU)~\cite{zeng2021adversarial}, and Detection-and-Purification (DePuD)~\cite{yan2021deep}\footnote{DePuD, a backdoor defense designed for SSL, erases the poisoned feature on images before training process.}. We evaluate the first four baseline methods using $5\%$ of the labeled data as extra clean data for defense purposes. we particularly note that UPure does not need any extra clean data for defense.

\noindent \textbf{Evaluation metrics.} We evaluate the effectiveness of backdoor defenses using two common metrics: Benign Accuracy (BA), which measures the model's accuracy on benign test data without a backdoor trigger, and Attack Success Rate (ASR), which evaluates the model's accuracy when tested on non-target class data with the specified trigger. Among all experimental results, the bold (underlined) numbers indicate the best (second best) results. We run all experiments for $200,000$ iterations. Our results are an average of three independent trials.

\subsection{Comparison with state-of-the-art (SOTA) Defense Methods}\label{expt:main}

\cref{tab:main_result} shows the performance of various backdoor defenses across different SSL algorithms on datasets, including CIFAR10, SVHN, STL10, and CIFAR100. Our main results were obtained using the ``\textit{Add perturbation}'' strategy. In~\cref{sec:discussion}, we present results from other strategies to justify our selection.

\input{tables/tab-abl-other}

\noindent \textbf{Preserves benign accuracy (BA).} 
With the advantage of strong augmentation applied to unlabeled data within SSL algorithm, UPure adds perturbations to high-frequency components of images, keeping competitive BA to enhance robustness against backdoor attacks. For instance, in CIFAR10 with FixMatch, we observe only a marginal decrease to 91.05\%. For STL10 datasets, MixMatch and UDA achieve higher BA than those of the backdoored model. This is because UPure can be regarded as a kind of data augmentation, which tends to improve the accuracy of SSL training. We also observe that the BA (ASR) of UPure in ReMixMatch slightly declines (increases). We speculate that distribution alignment in ReMixMatch affects our performance. Compared to other defenses, UPure exhibits effectiveness against backdoor attacks in SSL, while maintaining or even better model accuracy on benign samples. Importantly, without the need for additional clean data, UPure shows comparable performance to the original model in BA in most cases. Due to space limitations, we provide additional results in supplementary to show that UPure can preserve BA under different scenarios, such as training on a clean training data set and untargeted poisoning attacks.

\noindent \textbf{Outperforms attack success rate (ASR).} Across multiple datasets, UPure consistently demonstrates superior or competitive performance in reducing ASR, often achieving rates of zero nearly, indicating highly effective mitigation of backdoor attacks. For instance, in CIFAR10, UPure achieves the lowest ASR values (\ie, $0.00\%$), particularly in UDA, Fixmatch, and Flexmatch, significantly surpassing other methods. Similarly, in SVHN and STL10, UPure maintains $0.00\%$ ASR for several SSL algorithms. This means that UPure improves SSL robustness against backdoor attacks. While existing backdoor defenses achieve lower ASR in some cases  (\eg, $1.07\%$ and $0.99\%$ in FT and NAD, respectively), SOTA defense methods typically require additional clean data to mitigate the backdoor impact on suspicious models. It is worth noting that UPure does not need extra clean samples. We also report more experimental results in the supplementary.

\subsection{Comparison with other repetitive trigger}
According to the lessons~\cite{shejwalkar2023perils} on SSL, we further evaluate the effectiveness of UPure on existing repetitive triggers~\cite{wang2022invisible, li2023embarrassingly}. FTrojan~\cite{wang2022invisible}, an invisible backdoor trigger, injects perturbations on the frequency domain to make it easier to function during model convergence. CTRL~\cite{li2023embarrassingly} then is a variant version of FTrojan~\cite{wang2022invisible} in self-supervised learning. To validate the efficacy of UPure, we consider such trigger patterns and compare them to SOTA defenses.

As depicted in~\cref{tab:abl-other}, all defense approaches, except UPure, demonstrate higher ASR rates when faced with a visible trigger compared to an invisible one, even striving to maintain a higher BA. Given the greater challenge posed by the visible trigger, we prioritize their analysis in~\cref{expt:main}. Surprisingly, UPure achieves an ASR of $0.00\%$ for both types of attacks, indicating complete defense success, despite slight variations in BA. This underscores UPure is able to resist backdoor triggers that can arise in both supervised and self-supervised learning scenarios.



\section{Discussion and Limitations}\label{sec:discussion}


\noindent \textbf{Different quantizers in DCT spectrum.}
To understand the impact of different quantizers in the DCT spectrum, UPure considers three types of perturbation strategies in high-frequency components: ``\textit{Turn to zero},'' ``\textit{Replace from other},'' and ``\textit{Add perturbation}''. As shown in~\cref{tab:abl-perturb}, we observe that erasing the high-frequency components of the images still preserves the original BA, but leaves some ASR (\eg, 0.72\% in ``\textit{Turn to zero}''). The replacement strategy leads to a greater decline in BA compared with others. Finally, we find that the ``\textit{Add perturbation}'' strategy achieves a good trade-off between ASR and BA compared with the above two strategies. On the other hand, \cref{tab:abl-perturb} reports the fidelity of the images generated after these three strategies. We observe that ``\textit{Add perturbation}'' has higher PSNR and SSIM values compared with others. Therefore, we choose this strategy to minimize BA performance degradations in~\cref{expt:main}.

\input{tables/tab-abl-perturb}

\noindent \textbf{RDP Trade-offs of UPure.}
We delve deeper into the relationship between UPure and RDP trade-offs. Empirically, we measure the squared error as the distortion and Fréchet Inception Distance (FID)~\cite{heusel2017gans}(\ie, Wasserstein-2 distance) as the perception, respectively, as outlined in~\cite{zhang2021universal}.

In~\cref{tab:rdp-real}, we utilize $2,048$ samples from CIFAR10 as a toy example to compute the distortion and perception. The theoretical lower bound of rate, as given by RDP function for Gaussian source in~\cite{zhang2021universal}, indicates values of $1.76$, $1.43$, and $3.51$ for ``\textit{Turn to zero},'' ``\textit{Replace from other},'' and ``\textit{Add perturbation},'' respectively, under the setting of image size of $32\times32$ and $\tau=16$. However, the rate in our defense exceeds the theoretical lower bound across the three strategies. This indicates that there is still room for improvement in the design of quantizers. 

\input{tables/tab-discuss-rdf}

\begin{figure}[!htp]
    \centering
    \includegraphics[width=0.25\textwidth]{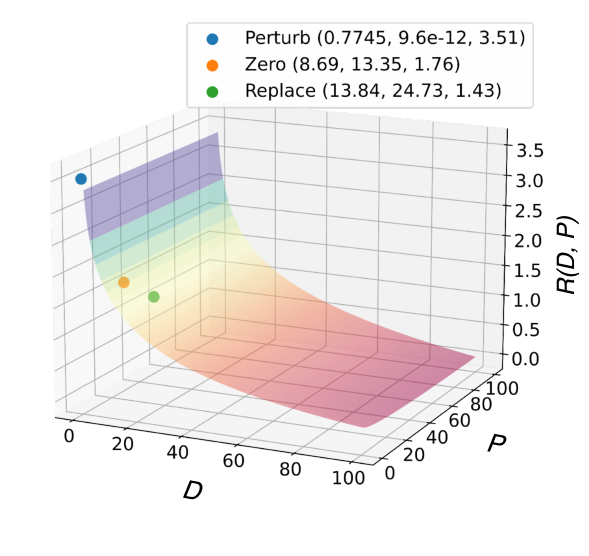}
    \vspace{-0.1in}
    \caption{Visualize RDP function of UPure on CIFAR10.}
    \vspace{-0.1in}
    \label{fig:5-discuss}
\end{figure}

As illustrated in~\cref{fig:5-discuss}, we depict the position of these three strategies on the RDP function. Specifically, ``\textit{Add perturbation}'' requires a higher rate, compared to the other two that have comparable rates. While the closed-form solution of the RDP function does not explicitly offer a practical design for the quantizer, it still allows us to verify and analyze the validity of rate we used. Here note that we use the RDP trade-offs to derive a minimum $\tau$ as an adaptive selection strategy. There indeed exists a feasible $\tau$ if we have a success probability in defense.


\noindent \textbf{Quantitative explanation of Lemma~\ref{lem:1} and Lemma~\ref{lem:2}.} In~\cref{fig:6-varylem1}, as the dimensions of $\mathsf{H}_c$ and $\mathsf{W}_c$ become higher, it will enhance the effectiveness of our defense mechanism, especially in scenarios, where the size of the backdoor trigger is unknown. This stems from the observation that larger trigger sizes $\mathsf{H}_t \times \mathsf{W}_t$ tend to simplify the defense process. Note that the minimum image size requirement is $32\times 32$, at least twice larger than the cutout region for unrestricted movement.

In~\cref{fig:7-lem2}, given $\mathsf{N}=16$, the behavior of $p_f^{repet}$ is influenced by $\mathsf{M}$, $\beta$, and $q_k$. The $x$-axis represents discrete values of $\beta$, ranging from $1$ to $11$, where $11$ comes from the minimum of $(\mathsf{N}-\mathsf{M}-1)$. The $y$-axis quantifies the probability $p_f^{repet}$, calculated for each $\beta$ value under different settings of $q_k$ and $\mathsf{M}$ (\ie, $q_k = 0.2, 0.5, 0.8$ and $\mathsf{M} = 2,4$). The markers on the lines (circles for $q_k = 0.2$, squares for $q_k = 0.5$, and triangles for $q_k = 0.8$) indicate the computed probabilities for discrete $\beta$ values. The dashed line represents different $\mathsf{M}$ values. From~\cref{fig:7-lem2}, we observe that higher $q_k$  and smaller $\mathsf{M}$ dominate the lower bound, with instances where it decreases to $0$ under certain conditions. 

\begin{figure}[!htp]
    \centering
    \begin{minipage}{0.48\linewidth}
        \centering        \includegraphics[width=\linewidth]{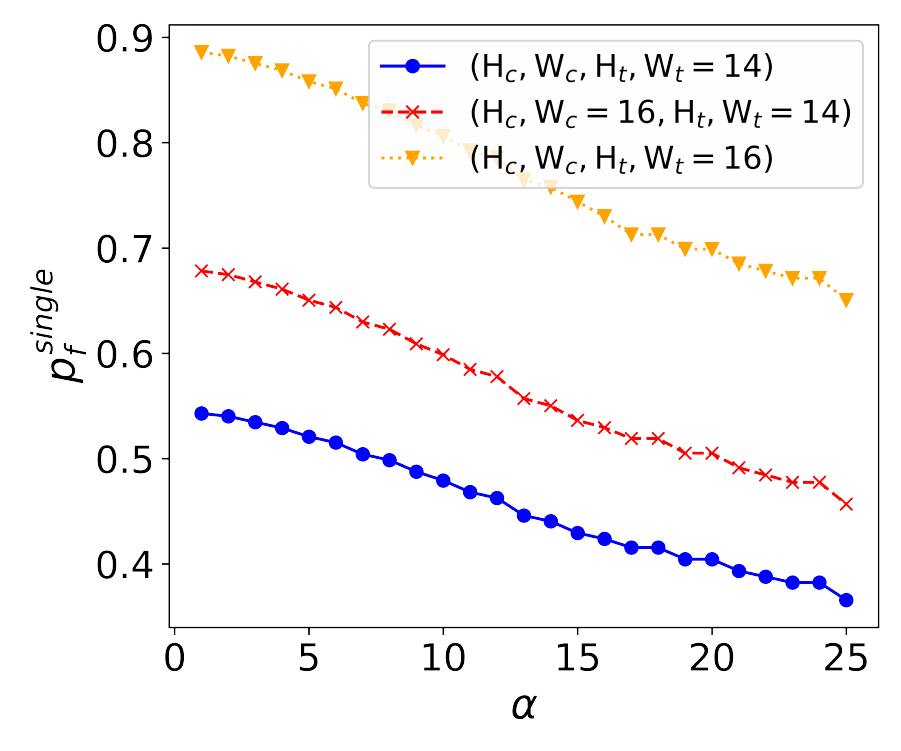}
        \vspace{-0.15in}
        \caption{The failure probability $p_f^{single}$ with different $\mathsf{H}_c$, $\mathsf{W}_c$, $\mathsf{H}_t$, $\mathsf{W}_t$, and $\alpha$.}
        \label{fig:6-varylem1}
    \end{minipage}
    \hspace{0.2ex}
    \begin{minipage}{0.48\linewidth}
        \centering        \includegraphics[width=\linewidth]{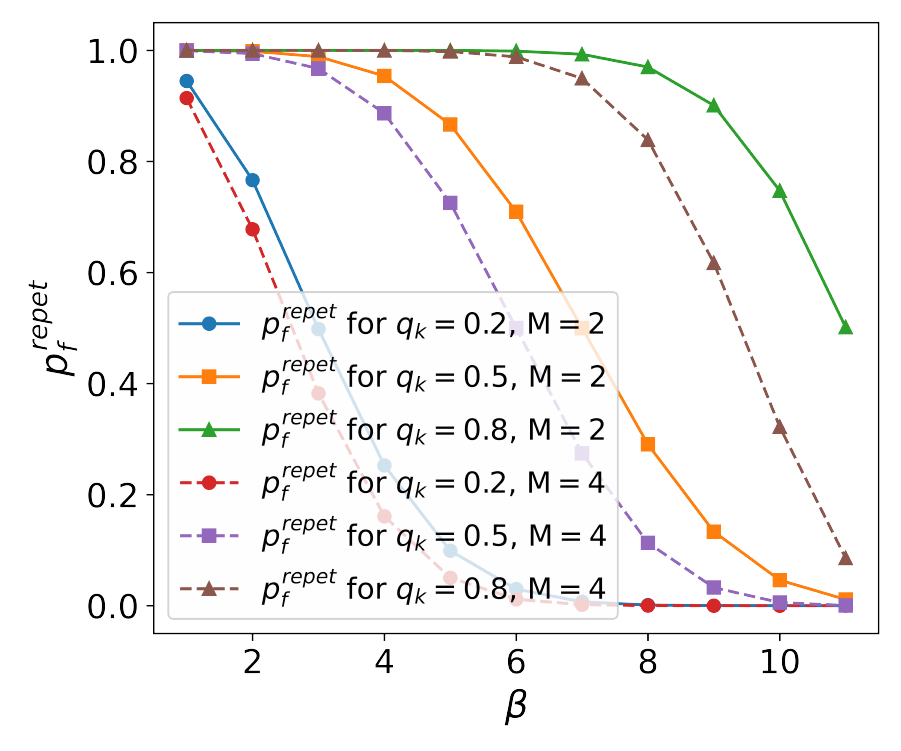}
         \vspace{-0.15in}
        \caption{The failure probability $p_f^{repet}$ with different $q_k$ and $\beta$.}
        \label{fig:7-lem2}
    \end{minipage}
\end{figure}

\vspace{-0.1in}
\section{Conclusions}

In this paper, we propose UPure to defend against backdoor attacks on SSL. UPure purifies certain frequency coefficients of the unlabeled training data and makes repetitive trigger patterns ineffective. Our study from the theoretical perspective indicates that we not only derive an appropriate perturbation region but also explain the proposed quantizers by the theorem. Our experiments show that UPure is effective and suppresses repetitive trigger patterns, outperforming existing defense methods. With this work, we put to light an important direction for defense against backdoor attacks in SSL scenarios.

\section*{Acknowledgements}
    This work was supported by the National Science and Technology Council (NSTC), 
    Taiwan, ROC, under Grant NSTC 112-2221-E-001-011-MY2. We thank to National Center for High-performance Computing (NCHC) for providing computational and storage resources.
    
\clearpage
{\small
\bibliographystyle{ieee_fullname}
\bibliography{egbib}
}

\newpage


\renewcommand{\thesection}{\Alph{section}}
\setcounter{section}{0}
\setcounter{page}{1}
\setcounter{equation}{5}
\setcounter{theorem}{1}
\setcounter{table}{4}
\setcounter{figure}{6}

The content of Supplementary Material is summarized as follows: 1) In~\cref{app:A}, we discuss the details of SSL algorithms used in our work; 2) In~\cref{app:B}, we state the implementation and training details we used in the experiment in terms of datasets, hyper-parameters, and model architectures to ensure that our method can be reproduced; 3) In~\cref{app:C}, we first recall the rate-distortion theory. Then, we present the rate-distortion-perception (RDP) trade-offs from theoretical derivation. Last, we discuss the justification that repetitive trigger patterns are ineffective.

\section{SSL Algorithms}\label{app:A}
We conducted the main experiments using five state-of-the-art SSL algorithms, briefly summarized as follows.

\textbf{(a) MixMatch}~\cite{berthelot2019mixmatch} creates various weakly augmented versions of each unlabeled sample. It then calculates the outputs of the current model for these versions and sharpens the average prediction by raising all its probabilities before normalization. This refined prediction acts as the label for the unlabeled sample. In addition, MixMatch employs mixup regularization on all training data and trains the model using cross-entropy loss.

\textbf{(b) ReMixMatch}~\cite{berthelot2019remixmatch} replaces weak data augmentation in MixMatch with AutoAugment and enhances consistency regularization through augmentation anchoring. This technique involves using predictions made on a weakly augmented version of an unlabeled sample as the target prediction for a strongly augmented version of the same sample. Additionally, it employs distribution alignment, which normalizes the new model predictions on unlabeled data using the running average of model predictions on unlabeled data, significantly enhancing the resulting model's performance.

\textbf{(c) Unsupervised data augmentation (UDA)}~\cite{xie2020unsupervised} exhibits superior performance on SSL with the benefit of strong data augmentations, such as RandAugment, instead of weak data augmentations used in MixMatch. Specifically, RandAugment randomly chooses a few powerful augmentations to improve the generalization and robustness of the model.

\textbf{(d) FixMatch}~\cite{sohn2020fixmatch} combines consistency regularization and pseudo-labeling while simplifying the complex ReMixmatch algorithms. In FixMatch, weak augmentation follows a standard flip-and-shift strategy, randomly flipping images horizontally with a given probability. For strong augmentation, RandAugment and CTAugment are employed.  Furthermore, Cutout is followed after these above operations.

\textbf{(e) FlexMatch}~\cite{zhang2021flexmatch} presents a curriculum pseudo-labeling (CPL) approach, which flexibly sets the threshold of pseudo-labels in different categories in each training iteration. Then, according to the model's learning status, FlexMatch selects more informative unlabeled data and their pseudo-labels.

\section{Experimental Details}\label{app:B}

\subsection{Datasets and DNNs}\label{app:B.1}
We list the detailed dataset and model architecture used in our experiments, as summarized in Table~\ref{tab:details_dataset}, which shows the number of classes in each dataset alongside the number of training and test data. Compared to CIFAR10, SVHN is designed for recognizing street-view house numbers and is not class-balanced. STL10 is a 10-class classification task specially designed for semi-supervised learning research. Besides, we use WideResNet~\cite{zagoruyko2016wide} as model architecture in our experiments.

\input{tables/tab-details_dataset}

\subsection{Training size of Labeled data}\label{app:B.label}
Table~\ref{tab:setup_data} shows that different sizes of labeled data depend on the algorithm in our experiments. Even when there is little labeled training data (\ie, 100 samples for CIFAR10), UPure still effectively alleviates backdoor effects and maintains model accuracy.


\input{tables/tab-setup_data}

\subsection{Training Settings}\label{app:B.3}
Following the training settings in~\cite{wang2022usb}, we adopted an SGD optimizer with a momentum of $0.9$, a weight decay of $1 \times 10^{-3}$, layer decay of $1$, a crop ratio of $0.875$, and an initial learning rate of $3 \times 10^{-2}$ in our experiments. With a batch size of $64$, we trained the WideResNet-28-2 model for $200,000$ iterations, as shown in Table~\ref{tab:details_dataset}. Note that for CIFAR100, we trained the WideResNet-28-8 model. All the other settings in SSL algorithms are the same as the original configurations in the USB package~\cite{wang2022usb}. Furthermore, we executed the experiments with a single NVIDIA RTX3090 GPU. Nevertheless, SSL is still cost-expensive for computation. For instance, in our experiments, the FixMatch algorithm requires approximately $16$ hours to complete $200,000$ iterations on CIFAR10, whereas, for MixMatch and ReMixMatch, each takes about $6$ hours. Training for the same number of iterations on CIFAR100 using FixMatch extends to $3.5$ days, leading us to exclude experiments with UDA and FlexMatch on CIFAR100.

\subsection{UPure algorithms}\label{app:B.4}
Our training procedure is described in Algorithm~\ref{algo: UPure}. Specifically, UPure allows a defender to purify the unlabeled training data for SSL in the frequency domain using three strategies. This approach emphasizes preprocessing data before training, rather than identifying backdoor samples within the model.

\input{algorithms/ours-alg}

\subsection{Details of Backdoor Defense}\label{app:B.5}
We evaluate our experiments using four post-processing and one in-processing backdoor defenses for comparison, briefly summarized as follows.

\textbf{(a) Fine-tuning} is a common baseline to alleviate pernicious behavior on the backdoored model using additional clean labeled data. In our experiments, we utilize the labeled training data of SSL algorithm for finetune.

\textbf{(b) Fine-pruning}~\cite{liu2018fine} first prunes the inactivated neurons of the last layer by benign data and then finetunes the model to prevent the activation of backdoors. 

\textbf{(c) Neural Attention Distillation (NAD)}~\cite{li2020neural} fine-tunes a teacher model on a subset of benign data and distills the knowledge of the fine-tuned model into backdoored model for purification.

\textbf{(d) Backdoor Adversarial Unlearning (I-BAU)}~\cite{zeng2021adversarial} purifies the backdoored model by using an implicit hyper-gradient, which facilitates the model convergence and the generalizability of robustness given a small number of clean data.

\textbf{(e) Detection-and-Purification (DePuD)}~\cite{yan2021deep} utilizes GradCam to detect suspicious region (\emph{i.e.}, backdoor triggers) in images. According to the model's attention, the purification operation employs differential privacy to alleviate the effects of poisoned images.

\subsection{Visualization results of UPure}\label{app:B.6}
\vspace{-0.05in}
We present visualization results of UPure in~\cref{fig:B-visual}, comparing clean and backdoored samples with purified samples obtained from our three strategies. We utilize a repetitive backdoor attack in our work, as detailed in~\cite{shejwalkar2023perils}, with pixel intensity, width, and gap set to $30$, $1$, and $1$, respectively. As can be seen from~\cref{fig:B-visual}, the samples generated from ``\textit{Turn to zero}'' and ``\textit{Replace from others}'' are blurry while the ``\textit{Add perturbation}'' strategy performs minimal disturbance on the backdoored samples, which is sharper than the other two. This indicates that ``\textit{Add perturbation}'' is better than the other two to maintain the original fidelity of images. Note that the quantitative results of UPure are shown in~\cref{tab:supp-metric}.



\input{tables/tab-supp-metric}

\begin{figure}[!ht]
    \centering
    \includegraphics[width=0.7\linewidth]{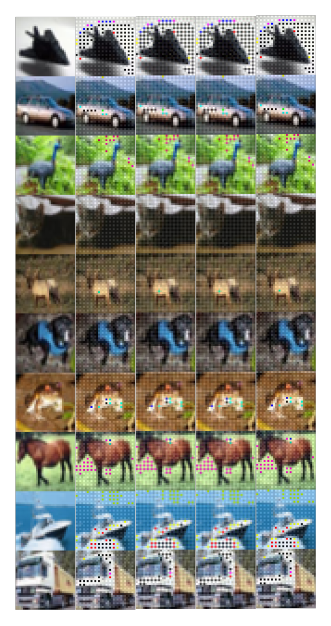}
    \caption{Visualization results on CIFAR10. (From Left to Right) The first and second columns display the clean and backdoored samples, respectively. The third, fourth, and fifth columns show purified samples obtained from  using  ``\textit{Turn to Zero},'' ``\textit{Replace from others},'' and ``\textit{Add Perturbation}'' strategies of UPure, respectively.}
    \label{fig:B-visual}
\end{figure}


\subsection{More experiment results}\label{app:B.7}

\subsubsection{Non-poisoned dataset applied UPure}\label{app:B.7.clean}

Since UPure can be viewed as a form of data augmentation, we further train a model using UPure on a clean training dataset. For example, with CIFAR-10 as the training set, MixMatch and FlexMatch attain high BA, close to their original performance, as shown in~\cref{tab:clean}. We find that ReMixMatch is more susceptible to high-frequency component perturbations, resulting in a significant decrease in BA. This observation implies that modern SSL algorithms can adapt UPure to protect unlabeled data and be robust to backdoor attacks in SSL scenarios.

\begin{table}[!ht]
\begin{center}
\caption{Evaluation on clean CIFAR10 dataset applied UPure.}
\label{tab:clean}
\resizebox{0.9\linewidth}{!}{%
\begin{tabular}{cccccccccc}
\toprule
 & MixMatch && ReMixMatch && UDA && FixMatch && FlexMatch \\ \midrule
BA & 89.31\% && 77.82\% && 88.24\% && 87.96\% && 94.18\% \\ \midrule
ASR & 0.21\% && 1.85\% && 0.34\% && 0.33\% && 0.26\% \\ \bottomrule
\end{tabular}}
\end{center}
\end{table}
\vspace{-0.2cm}

\subsubsection{Non-targeted attacks}\label{app:B.8.non-target}

Non-targeted attacks refer to the accuracy of classifying clean and trigger inputs based on a trained model. As the clean and poisoned datasets tend to have different class distributions, we consider non-targeted attacks and observe the model's accuracy drop to measure the attack effectiveness if the trigger is input. \cref{tab:nontarget} shows the purified model's accuracy in classifying clean and trigger inputs. More precisely, we find that even if the validation set does not contain the target class, the trigger inputs tend to be misclassified to certain classes (\eg, ``bird'', ``ship'', and ``airplane'' in the test set). We conduct five SSL algorithms trained on CIFAR10 that are considered in~\cref{tab:nontarget}. The purified model trained using FixMatch achieves $89.28$\% and $89.70$\% accuracy on clean and trigger inputs, respectively. However, the non-targeted BA of UDA algorithms drops more than the other four algorithms. 


\begin{table}[!ht]
\begin{center}
\caption{Evaluation against non-targeted attacks on CIFAR10.}
\label{tab:nontarget}
\resizebox{\linewidth}{!}{%
\begin{tabular}{cccccccccc}
\toprule
 & MixMatch && RemMixMatch && UDA && FixMatch && FlexMatch \\ \midrule
BA & 87.85\% && 87.22\% && 93.59\% && 89.28\% && 94.27\% \\ 
Non-target BA & 84.47\% && 83.73\% && 77.47\% && 89.70\% && 92.14\% \\ \bottomrule
\end{tabular}}
\end{center}
\end{table}
\vspace{-0.10in}

\subsubsection{Impact of different perturbation area sizes.}  We evaluate the impact of different perturbation area sizes, $\tau\times\tau$, in terms of BA and ASR on UPure in~\cref{fig4-b}, where $\tau \in \{4, 8, 16, 24 \}$. We find that a larger distortion region (\eg, $24\times24$) in the DCT spectrum sacrifices a little BA (\ie, a decrease of $2.33\%$ in BA) but preserves a better ASR (\ie, $0\%$). \cref{fig4-b} also indicates that a small perturbation region in the high-frequency component is not enough to remove the backdoors. To better compromise between BA and ASR, we choose an appropriate size (\ie, $\tau = 16$) in Sec.~4 for comparisons. 

\subsubsection{Impacts on poisoning rate of unlabeled data.} 
\cref{fig4-a} displays the performance of UPure under different poisoning rates of unlabeled data. Specifically, we vary the poisoning rates, \ie, $0.15\%$, $0.2\%$, $0.3\%$, $0.4\%$, and $0.5\%$. Regardless of the poisoning rates, UPure can suppress the occurrence of backdoors in terms of nearly zero ASRs and no significant drops in BA. This implies that UPure successfully breaks the association between the backdoor and target class with minimal changes to BA. Due to resource constraints, we perform these experiments only for a subset of combinations from Sec.~4. 

\begin{figure}[!htp]
\centering
    \begin{minipage}{0.335\linewidth}
        \centering
        \includegraphics[width=0.95\linewidth]{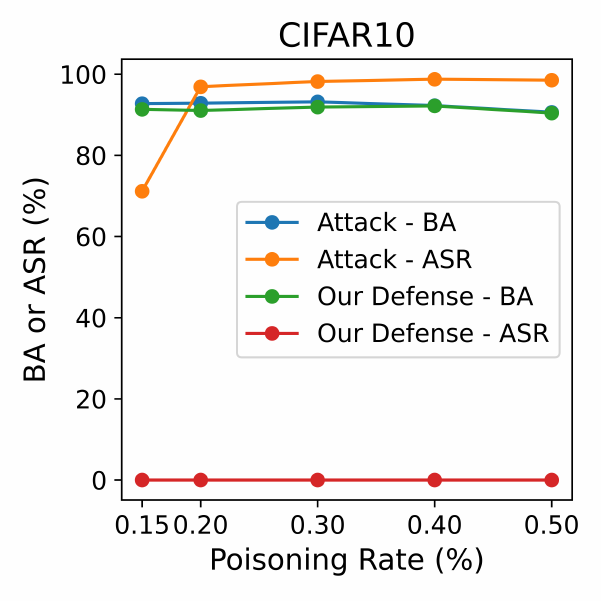}
        \caption{Poisoning rate vs. BA/ASR.}
        \label{fig4-a}
    \end{minipage}
    \begin{minipage}{0.64\linewidth}
        \centering
        \includegraphics[width=\linewidth]{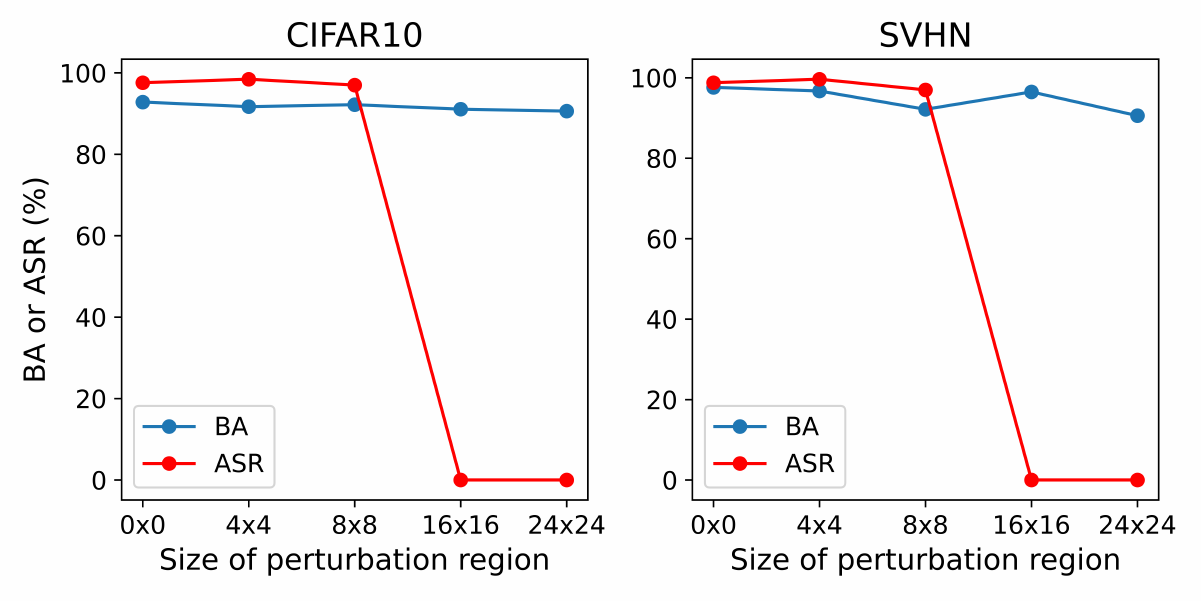}
        \caption{Different sizes of perturbation regions vs. BA/ASR.}
        \label{fig4-b}
    \end{minipage}
\end{figure}

\subsection{Comparison with pre-processing methods} 

We compare UPure with three filter-based data pre-processing methods (\ie, Gaussian Filter, Bilateral Filter~\cite{tomasi1998bilateral}, and Median Filter). We preprocess the unlabeled training data with these filters before feeding them into the model. From~\cref{tab:abl-smooth}, we observe that Gaussian Filter can be effective in terms of lowering ASR across CIFAR10 and SVHN. However, it also significantly degenerates the BA performance in CIFAR10. Bilateral Filter does not eliminate the backdoor effects in both datasets. Median Filter reduces BA the most among other methods. UPure is more effective at resisting backdoor attacks than other methods, showing only a minor decrease in BA and a significantly low ASR.

\input{tables/tab-abl-smooth}

\section{Theory Details}\label{app:C}

\subsection{Rate-Distortion Theory}\label{app:C.1-rd}
Rate-distortion theory analyzes the fundamental trade-off between the rate used for representing samples from a data source $X\sim p_X$, and the expected distortion incurred in decoding those samples from their compressed representations. Formally, the relation between the input $X$ and output $\hat{X}$ of an encoder-decoder pair, is a (possibly stochastic) mapping defined by some conditional distribution $p_{\hat{X}|X}$. The expected distortion of the decoded signals is thus defined as
\begin{align}
    \mathbb{E}[\Delta(X,\hat{X})],
\end{align}
where the expectation is \wrt the joint distribution $p_{X,\hat{X}}=p_{\hat{X}|X}p_X$, and $\Delta:\mathcal{X} \times \hat{\mathcal{X}} \rightarrow \mathbb{R}^+$ is any full-reference distortion measure such that $\Delta(x,\hat{x})=0$ if and only if $x=\hat{x}$.

A key result in rate-distortion theory states that for an \iid source $X$, if the expected distortion is bounded by $D$, then the lowest achievable rate $R$ is characterized by the (information) rate-distortion function

\begin{equation}\label{eq:RD}
R(D) = \min_{p_{\hat{X}|X}} \, I(X,\hat{X}) \quad \textrm{s.t.} \quad \mathbb{E}[\Delta(X,\hat{X})] \le D,
\end{equation}

where $I$ denotes mutual information \cite{cover2012elements}. 
Closed-form expressions for the rate-distortion function $R(D)$ are known for only a few source distributions and under simple distortion measures (\eg, squared error or Hamming distance). However several general properties of this function are known, including that it is always monotonic, non-increasing, and convex. 

\subsection{The Rate-Distortion-Perception Trade-offs}\label{app:C.2-rdp-tf}



In this section, we introduce the RDP function and its solution for Gaussian sources. While there are various solutions under different source conditions, we focus on the Gaussian version due to its simplification of mathematical treatment, particularly under mean squared error (MSE) distortion. This allows us to conduct a theoretical analysis of the universal RDP representation.

\begin{figure}[!ht]
    \centering
    \begin{subfigure}[b]{0.48\linewidth}
        \centering
        \includegraphics[width=\linewidth]{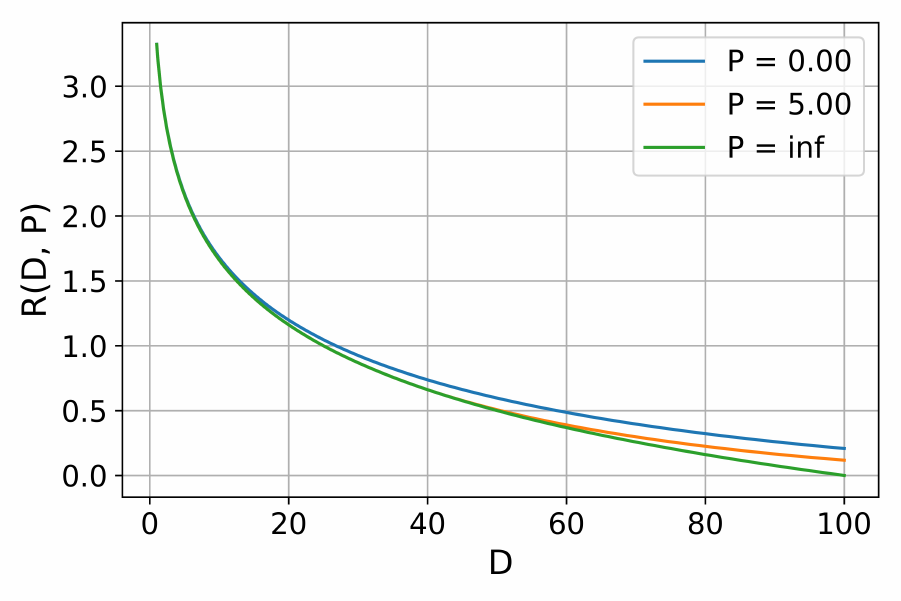}
        \caption{$R(D)$ under various $P$}
        \label{fig3-a}
    \end{subfigure}
    \begin{subfigure}[b]{0.48\linewidth}
        \centering
        \includegraphics[width=\linewidth]{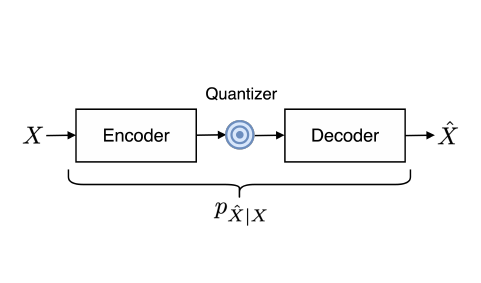}
        \caption{Lossy compression}
        \label{fig3-b}
    \end{subfigure}
    
    \caption{Rate-Distortion-Perception functions and lossy compression scheme. (a) The curve is computed from Eq.~(\ref{eq:grdp}) in Theorem~\ref{thm:gaussian_rdp}. Note that these curves represent the lower bound, with all points above them being considered feasible solutions.}
    \label{fig:rate_distortion_curves}
\end{figure}

\input{math_equations/rdp-tradeoff}

As shown in~\cref{fig3-a}, we can see that under the same distortion, the smaller the $P$ is, the larger the rate $R$ is. A source signal $X\!\sim\! p_X$ is mapped into a coded sequence by an encoder and back into an estimated signal $\hat{X}$ by the decoder. Through this concept of RDP, we aim to satisfy three properties in our approach: ($i$) if the coded sequence has low rate, it implies that more repetitive trigger patterns are eliminated; ($ii$) the reconstruction $\hat{X}$ is similar to the source $X$ on average (low distortion), which indicates that even the clean data is reconstructed well; ($iii$) the distribution $p_{\hat{X}}$ is similar to $p_X$ so that decoded signals are perceived as genuine source signals (good perceptual quality), which means that the model can still learn the original distribution well (\eg, high classification accuracy).

In addition, we also recall the concept of lossy compression, including an encoder, a decoder, and a quantizer, as depicted in~\cref{fig3-b}. The input and output signals follow a predefined distribution. In short, the lossy compression scheme can be viewed as a conditional probability distribution that can be analyzed in a mathematical treatment. For further details, please refer to the previous works~\cite{shannon1959coding,blau2018perception,blau2019rethinking,zhang2021universal}.

\subsection{Analysis of Theorem 1.}\label{app:thm:2}
Since the common SSL algorithms in~\cref{app:A} adapt cutout operation as strong data augmentation on unlabeled training data, we devise Lemma 1 to explain the failure of a single trigger if any area $\alpha$ of the trigger intersects with the cutout region. We further present Lemma 2 to discuss how UPure can resist repetitive trigger patterns by perturbing high-frequency components. Their proofs are presented below.


\subsubsection{Proof of Lemma 1.}


    Let $p_f^{single}$ denote the failure probability of a single trigger when a randomly positioned cutout region intersects it. It can be observed that $p_f^{single}$  reaches its minimum value when the trigger is placed at any corner of the image. This phenomenon can be attributed to the geometric constraints imposed by the image boundaries, which limit the spatial configurations available for the cutout region to intersect the triggers positioned at a corner.
    

Without loss of generality, we assume the trigger pattern is at the bottom-left corner of the image, where the coordinate is $O=(0,0)$. Let the position of bottom-left corner of the cutout region be $(w, h)$, as shown in~\cref{fig:C-lem1}. Under the constraint of minimal coverage area $\alpha$ that makes the trigger invalid, we can obtain the equation
\begin{align*}
    (\mathsf{W}_t - w)(\mathsf{H}_t - h) = \alpha.
\end{align*}

\begin{figure}[!ht]
    \centering
    \begin{minipage}{0.44\linewidth}
        \centering
        \includegraphics[width=\linewidth]{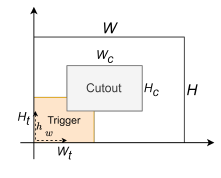}
        \caption{Example of cutout operation intersects with a trigger.}
        \label{fig:C-lem1}
    \end{minipage}
    \hspace{0.1ex}
    \begin{minipage}{0.43\linewidth}
        \centering
        \includegraphics[width=0.82\linewidth]{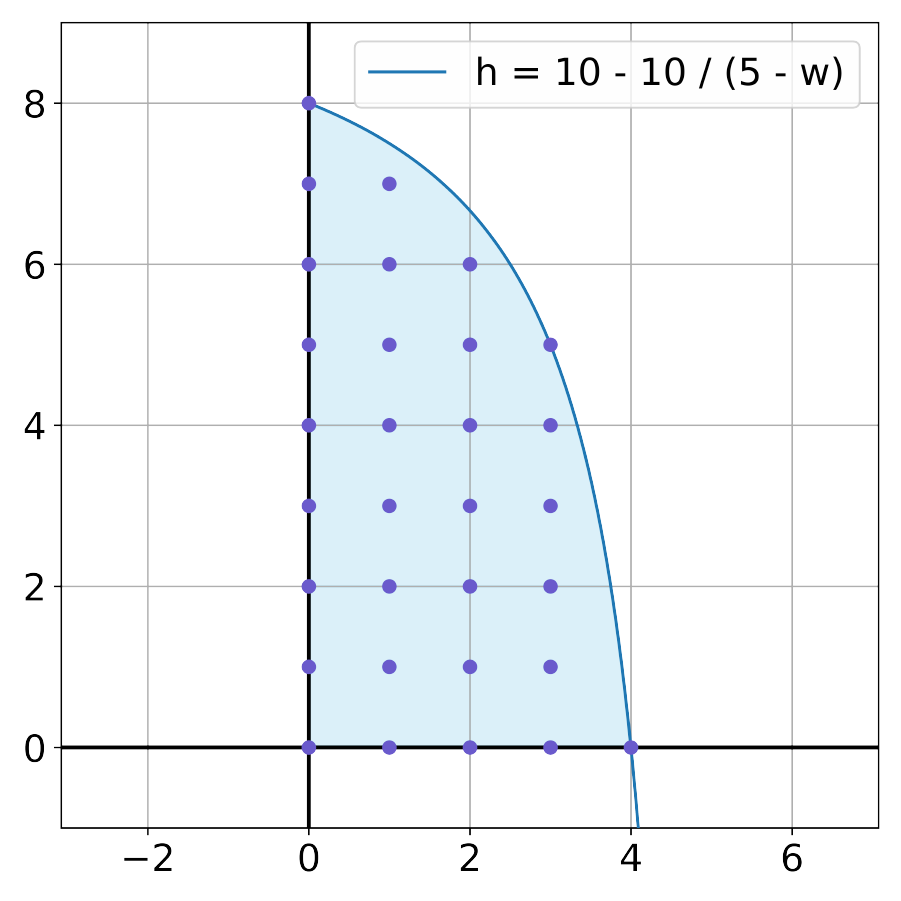}
        \caption{Example of lattice points under a hyperbolic.}
        \label{fig:C-hypo}
    \end{minipage}
\end{figure}

With $\alpha$ being a constant, we can derive $h$ and $w$ as follows:

1. Solving for $h$, we obtain:
\[
h = \mathsf{H}_t - \frac{\alpha}{(\mathsf{W}_t - w)}.
\]

2. Solving for $w$, we find:
\[
w = \mathsf{W}_t - \frac{\alpha}{(\mathsf{H}_t - h)}.
\]

Let us define a function $\Phi(w)$ such that:
\[
\Phi(w) = \mathsf{H}_t - \frac{\alpha}{(\mathsf{W}_t - w)}.
\]

Analyzing the roots of $\Phi(w)$, specifically when $\Phi(w) = 0$, yields:
\[
\mathsf{H}_t = \frac{\alpha}{(\mathsf{W}_t - w)} \Rightarrow w = \mathsf{W}_t - \frac{\alpha}{\mathsf{H}_t}.
\]

Additionally, evaluating $\Phi(w)$ at $w = 0$ gives:
\[
\Phi(0) = \mathsf{H}_t - \frac{\alpha}{\mathsf{W}_t}.
\]
Given the function \(\Phi(w) = \mathsf{H}_t - \frac{\alpha}{(\mathsf{W}_t - w)}\) (\eg, the blue curve in~\cref{fig:C-hypo}), we aim to compute the number of integer solutions under the curve defined by \(\Phi(w)\), which means every $(w,h)$ makes overlapping region $\textsf{Area}_\textsf{overlap} \geq \alpha$. The number of feasible configurations of $(w, h)$, in terms of the sum of integer solutions, can be calculated as follows:
\begin{align}\label{eq:int_sol}
    \textsf{\#config} = \sum_{w=0}^{\left\lfloor \mathsf{W}_t - \frac{\alpha}{\mathsf{H}_t} \right\rfloor}\left[\left\lfloor \Phi(w) \right\rfloor + 1\right].
\end{align}

Finally, the probability that the intersection $\textsf{Area}_\textsf{overlap}$ of the cutout region and trigger region is greater than $\alpha$ is:
\[
\resizebox{\linewidth}{!}{
$p_f^{single} \geq \text{Pr} \left( \textsf{Area}_{\textsf{overlap}} \geq \alpha \right) = \frac{\sum_{w=0}^{\left\lfloor W_t - \frac{\alpha}{H_t} \right\rfloor}\left[\left\lfloor \Phi(w) \right\rfloor + 1\right]}{(\mathsf{H}-\mathsf{H}_c+1)(\mathsf{W}-\mathsf{W}_c+1)}
$},
\]
where the denominator $(\mathsf{H}-\mathsf{H}_c+1) (\mathsf{W}-\mathsf{W}_c+1)$ represents the total number of distinct configurations in which the cutout region can manifest. 

\noindent \textbf{Remark.} In~\cref{fig:C-lem1}, we illustrate the movement of the cutout region when the trigger region is positioned at the bottom-left corner. In addition, in~\cref{fig:C-hypo}, we provide an example of computing ~\cref{eq:int_sol}, where each point in the area represents a feasible integer solution of $(w, h)$, the bottom-left corner of cutout region. Similarly, as discussed in Sec. 5 in our paper, we visualize the results of Lemma 1. 

Given that the failure probability of a single trigger pattern is bounded in Lemma 1, we can extend our analysis to repetitive patterns and explore methods to bound their failure probability in Lemma 2.

\subsubsection{Proof of Lemma 2.}
    As mentioned in~\cite{shejwalkar2023perils}, cutout~\cite{devries2017improved} operations and other strong data augmentations also destroy low frequency in SSL training. Here, we only consider the high-frequency components from $\mathsf{M+1}$ to $\mathsf{N}-1$, the total number of changeable coefficients is $\mathsf{N}-\mathsf{M}-1$. The first combination number should be $\binom{\mathsf{N}-\mathsf{M}-1}{\beta}$, and it's important to account for numbers greater than $\beta$ as well. The proof is concluded.

\end{document}

%% file: tables/tab-main_result.tex
\begin{table*}[!ht]
\begin{center}
\caption{Evaluation of various backdoor defense across different SSL algorithms on four datasets. All values are showed as percentage.}\label{tab:main_result}
\vspace{-0.1in}
\centering
\begin{subtable}{\textwidth}
\centering
\caption{CIFAR10}
\scalebox{0.78}{
\begin{tabular}{ccccccccccccccccccccc}
\toprule
\multirow{2}{*}{Algorithm} & \multicolumn{2}{c}{No Defense} & &\multicolumn{2}{c}{FT} & &\multicolumn{2}{c}{FP~\cite{liu2018fine}} &  & \multicolumn{2}{c}{NAD~\cite{li2020neural}} &  &\multicolumn{2}{c}{I-BAU~\cite{zeng2021adversarial}} &  & \multicolumn{2}{c}{DePuD~\cite{yan2021deep}} & & \multicolumn{2}{c}{UPure} \\ 
\cmidrule{2-3} \cmidrule{5-6} \cmidrule{8-9} \cmidrule{11-12} \cmidrule{14-15} \cmidrule{17-18} \cmidrule{20-21}
& \multicolumn{1}{c}{BA} & ASR & &\multicolumn{1}{c}{BA} & ASR & &\multicolumn{1}{c}{BA} & ASR & & \multicolumn{1}{c}{BA} & ASR & & \multicolumn{1}{c}{BA} & ASR && \multicolumn{1}{c}{BA} & ASR && \multicolumn{1}{c}{BA} & ASR \\
\midrule
Mixmatch~\cite{berthelot2019mixmatch} & 92.71&98.14&&91.40&\underline{2.17}&&\underline{91.82}&96.17&&88.39&2.30&&\textbf{91.85}&96.29&&84.99&2.46&& 87.57&\textbf{0.20}\\
\midrule
Remixmatch~\cite{berthelot2019remixmatch} & 87.86&83.65&&\textbf{87.94}&\textbf{1.78}&&\underline{87.82}&88.60&&87.87&\underline{1.84}&&87.86&88.88&&66.76&98.10&&85.30&29.62\\ 
\midrule
UDA~\cite{xie2020unsupervised} & 92.76&98.13&&92.96&3.09&&\underline{93.22}&96.90&&92.97&\underline{2.61}&&\textbf{93.31}&97.97&&87.62&87.71&&90.91&\textbf{0.00}\\
\midrule
Fixmatch~\cite{sohn2020fixmatch} &92.80&97.54&&92.36&3.62&&\underline{92.71}&95.00&&92.33&\underline{3.28}&&\textbf{92.96}&97.50&&85.24&94.58&&91.05&\textbf{0.00}\\
\midrule
Flexmatch~\cite{zhang2021flexmatch} & 94.42&97.16&&93.78&\underline{2.11}&&\underline{94.30}&96.94&&93.90&2.80&&\textbf{94.42}&97.17&&87.42&86.33&&93.66&\textbf{0.00}\\
\bottomrule
   \end{tabular}}
\end{subtable}

\begin{subtable}{\textwidth}
\centering
\vspace*{.3em}
\caption{SVHN}
\scalebox{0.78}{
\begin{tabular}{ccccccccccccccccccccc}
\toprule
\multirow{2}{*}{Algorithm} & \multicolumn{2}{c}{No Defense} & &\multicolumn{2}{c}{FT} & &\multicolumn{2}{c}{FP~\cite{liu2018fine}} &  & \multicolumn{2}{c}{NAD~\cite{li2020neural}} &  &\multicolumn{2}{c}{I-BAU~\cite{zeng2021adversarial}} &  & \multicolumn{2}{c}{DePuD~\cite{yan2021deep}} & & \multicolumn{2}{c}{UPure} \\ 
\cmidrule{2-3} \cmidrule{5-6} \cmidrule{8-9} \cmidrule{11-12} \cmidrule{14-15} \cmidrule{17-18} \cmidrule{20-21}
& \multicolumn{1}{c}{BA} & ASR & &\multicolumn{1}{c}{BA} & ASR & &\multicolumn{1}{c}{BA} & ASR & & \multicolumn{1}{c}{BA} & ASR & & \multicolumn{1}{c}{BA} & ASR && \multicolumn{1}{c}{BA} & ASR && \multicolumn{1}{c}{BA} & ASR \\
\midrule
Mixmatch~\cite{berthelot2019mixmatch} & 95.91 & 88.74 &&95.12&\textbf{0.41}&& \underline{95.55}&94.89&&90.22&\underline{1.05}&&\textbf{95.56}&94.93&&64.17&3.47&&92.28&1.09\\
\midrule
Remixmatch~\cite{berthelot2019remixmatch} & 92.19 &49.96 &&90.99&\underline{2.95}&&\textbf{92.59}&53.62&&90.43&3.57&&\underline{92.49}&52.85&&87.91&7.64&&90.80&\textbf{1.96}\\ 
\midrule
UDA~\cite{xie2020unsupervised} & 94.21 & 97.25 &&96.25&1.23&&\underline{97.72}&99.81&&96.47&90.37&&\textbf{97.78}&99.81&&96.44&1.12&&95.78&\textbf{0.00}\\
\midrule
Fixmatch~\cite{sohn2020fixmatch} & 97.61 & 98.76 &&95.97&2.15&&\underline{97.57} &99.18&&96.25&\underline{1.57}&&\textbf{97.67}&99.10&&96.86&\underline{0.64}&&96.49&\textbf{0.00}\\
\midrule
Flexmatch~\cite{zhang2021flexmatch} &93.90&98.82 &&93.78&\underline{2.27}&&\underline{95.29}&99.51&&93.32&2.42&&\textbf{95.56}&99.48&&92.38&11.16&&93.45&\textbf{0.00}\\
\bottomrule
\end{tabular}}
\end{subtable}

\begin{subtable}{\textwidth}
\centering
\vspace*{.3em}
\caption{STL10}
\scalebox{0.78}{
\begin{tabular}{ccccccccccccccccccccc}
\toprule
\multirow{2}{*}{Algorithm} & \multicolumn{2}{c}{No Defense} & &\multicolumn{2}{c}{FT} & &\multicolumn{2}{c}{FP~\cite{liu2018fine}} &  & \multicolumn{2}{c}{NAD~\cite{li2020neural}} &  &\multicolumn{2}{c}{I-BAU~\cite{zeng2021adversarial}} &  & \multicolumn{2}{c}{DePuD~\cite{yan2021deep}} & & \multicolumn{2}{c}{UPure} \\ 
\cmidrule{2-3} \cmidrule{5-6} \cmidrule{8-9} \cmidrule{11-12} \cmidrule{14-15} \cmidrule{17-18} \cmidrule{20-21}
& \multicolumn{1}{c}{BA} & ASR & &\multicolumn{1}{c}{BA} & ASR & &\multicolumn{1}{c}{BA} & ASR & & \multicolumn{1}{c}{BA} & ASR & & \multicolumn{1}{c}{BA} & ASR && \multicolumn{1}{c}{BA} & ASR && \multicolumn{1}{c}{BA} & ASR \\
\midrule
Mixmatch~\cite{berthelot2019mixmatch} & 84.89&79.90&&86.44&8.93&&84.20&93.99&&\underline{86.51}	&8.56&&84.31&93.68&&84.92&\underline{2.53}&&\textbf{87.45}&\textbf{0.25}\\
\midrule
Remixmatch~\cite{berthelot2019remixmatch} &90.89&74.88&&89.35&4.74&&\textbf{90.21}&75.47&&89.39&4.49&&\underline{90.19}&73.88&&84.54&\underline{1.75}&&81.66&\textbf{1.71}\\ 
\midrule
UDA~\cite{xie2020unsupervised} & 90.12&96.33&&86.09&1.63&&88.11&99.60&&82.36&0.94&&87.41&99.57&&88.78&\underline{0.15}&&\textbf{91.07}&\textbf{0.00}\\ 
\midrule
Fixmatch~\cite{sohn2020fixmatch} &91.84&99.56&&\underline{91.49}&2.51&&\textbf{91.95}&99.63&&84.53&2.29&&90.85&99.58&&87.71&\underline{0.35}&&91.22&\textbf{0.00}\\ 
\midrule
Flexmatch~\cite{zhang2021flexmatch} &90.46&99.78&&90.04&6.08&&90.19&99.81&&89.26&4.76&&\underline{90.31}&99.74&&87.88&\underline{0.40}&&\textbf{90.95}&\textbf{0.00}\\ 
\bottomrule
\end{tabular}}
\end{subtable}

\begin{subtable}{\textwidth}
\centering
\vspace*{.3em}
\caption{CIFAR100}
\scalebox{0.78}{
\begin{tabular}{ccccccccccccccccccccc}
\toprule
\multirow{2}{*}{Algorithm} & \multicolumn{2}{c}{No Defense} & &\multicolumn{2}{c}{FT} & &\multicolumn{2}{c}{FP~\cite{liu2018fine}} &  & \multicolumn{2}{c}{NAD~\cite{li2020neural}} &  &\multicolumn{2}{c}{I-BAU~\cite{zeng2021adversarial}} &  & \multicolumn{2}{c}{DePuD~\cite{yan2021deep}} & & \multicolumn{2}{c}{UPure} \\ 
\cmidrule{2-3} \cmidrule{5-6} \cmidrule{8-9} \cmidrule{11-12} \cmidrule{14-15} \cmidrule{17-18} \cmidrule{20-21}
& \multicolumn{1}{c}{BA} & ASR & &\multicolumn{1}{c}{BA} & ASR & &\multicolumn{1}{c}{BA} & ASR & & \multicolumn{1}{c}{BA} & ASR & & \multicolumn{1}{c}{BA} & ASR && \multicolumn{1}{c}{BA} & ASR && \multicolumn{1}{c}{BA} & ASR \\
\midrule
Mixmatch~\cite{berthelot2019mixmatch} & 70.50&93.16&&68.94&\underline{0.66}&&68.90&92.35&&67.69&0.68&&\textbf{69.18}&92.23&&65.17&4.56&&\underline{69.03}&\textbf{0.03}\\
\midrule
Remixmatch~\cite{berthelot2019remixmatch} & 72.60&94.80&&71.57&\underline{0.34}&&\textbf{72.13}&96.86&&70.85&0.40&&\underline{72.12}&96.74&&53.05&7.82&&61.82&\textbf{0.01}\\ 
\midrule
Fixmatch~\cite{sohn2020fixmatch} & 70.44&99.24&&68.97&1.07&&69.80&99.28&&56.99&0.99&&\underline{70.02}&99.24&&60.36&\underline{0.02}&& \textbf{71.04}&\textbf{0.00}\\ 
\bottomrule
\end{tabular}}
\end{subtable}
\end{center}
\vspace{-0.2in}
\end{table*}

%% file: tables/tab-abl-other.tex
\setlength{\tabcolsep}{4pt}
\begin{table*}[!ht]
\begin{center}
\caption{Comparison of defense against other repetitive triggers on CIFAR-10.}
\vspace{-0.05in}
\label{tab:abl-other}
\scalebox{0.85}{
\begin{tabular}{c|c|ccccccc}
\toprule
Attack & Metric & No Defense & FT & FP & NAD & I-BAU & DePuD & Ours \\
\midrule
\multirow{2}{*}{Invisible~\cite{li2023embarrassingly,wang2022invisible}} &BA&94.38&93.63&\underline{94.27}&93.62&\textbf{94.38}&88.69&92.89\\
& ASR &91.20 & 3.88&1.40&3.88&1.23&\underline{1.16}&\textbf{0.00}\\
\midrule	
\multirow{2}{*}{Visible~\cite{shejwalkar2023perils}} &BA &92.80& \underline{92.36}&92.71&92.33&\textbf{92.96}&85.24&91.05\\
& ASR &97.54&3.62&95.00&\underline{3.28}&97.50&94.58&\textbf{0.00}\\
\bottomrule
\end{tabular}}
\end{center}
\vspace{-0.2in}
\end{table*}
\setlength{\tabcolsep}{1.4pt}

%% file: tables/tab-abl-perturb.tex
\setlength{\tabcolsep}{4pt}
\begin{table}[!ht]
\begin{center}
\caption{Comparison with our three strategies in FixMatch SSL algorithms.}
\label{tab:abl-perturb}
\scalebox{0.85}{
\begin{tabular}{lcccc}
\toprule
\multirow{2}{*}{Strategy} & \multicolumn{4}{c}{CIFAR10}\\
\cmidrule{2-5}
& \multicolumn{1}{c}{BA} & ASR  & PSNR & SSIM\\
\midrule
Turn to zero &90.52&0.72&33.03&$0.9618{\pm.075}$\\
Replace from other &88.86&0.00&30.65&$0.9461{\pm.088}$\\
Add perturbation &91.05&0.00&45.43&$0.9969{\pm.012}$\\
\bottomrule
\end{tabular}}
\end{center}
\vspace{-0.1in}
\end{table}

\setlength{\tabcolsep}{1.4pt}

%% file: tables/tab-discuss-rdf.tex
\vspace{-0.05in}
\begin{table}[!htp]
\centering
\caption{The exact value of RDP of our defense.}
\vspace{-0.05in}
\scalebox{0.85}{
\begin{tabular}{lccc}
\toprule
Strategy  & Distortion $\downarrow$ & Perception $\downarrow$ & Rate lower bound\\
\midrule
Turn to zero& 8.69&13.35& 1.76\\
Replace from other&13.84& 24.73&1.43\\
Add perturbation& 0.7745&9.6e-12&3.51\\
\bottomrule
\end{tabular}}
\label{tab:rdp-real}
\end{table}

%% file: tables/tab-details_dataset.tex
\begin{table*}[!htbp]
\caption{Statistics of datasets used in our experiments}\label{tab:details_dataset}
\centering
\scalebox{0.85}{
\begin{tabular}{llclclclclc}
  \toprule
  Dataset &&Input size &&Classes && Unlabeled Training data &&Test data && Model \\ 
 \midrule \midrule
  CIFAR10 &&$3\times32\times32$ && 10 && 50000 && 10000 &&WideResNet-28-2 \\  
  SVHN && $3\times32\times32$ && 10 && 73257 && 26032 &&WideResNet-28-2 \\
  STL10 && $3\times96\times96$ && 10 && 100000 && 1000 && WideResNet-28-2 \\
  CIFAR100 && $3\times32\times32$ && 100 && 50000 && 10000 &&WideResNet-28-8 \\ 
  \bottomrule
\end{tabular}}
\end{table*}

%% file: tables/tab-setup_data.tex
\begin{table}[!ht]
\caption{Sizes of labeled data for training a model across different SSL algorithms.}\label{tab:setup_data}
\centering
\scalebox{0.75}{
\begin{tabular}{lccccc}
  \toprule
  \multirow{2}{*}{Dataset} & \multicolumn{5}{c}{Algorithm} \\ 
  \cmidrule{2-6}
  & MixMatch & ReMixMatch & UDA & FixMatch & FlexMatch \\  \midrule
  CIFAR10 & 4000 & 100 & 100 & 100 & 100 \\  
  SVHN & 250 & 250 & 100 & 100 & 100 \\
  STL10 & 3000 & 1000 & 1000 & 1000 & 1000 \\
  CIFAR100 & 10000 & 2500 & 2500 & 2500 & 2500 \\ 
  \bottomrule
\end{tabular}}
\vspace{-0.1in}
\end{table}

%% file: algorithms/ours-alg.tex
\newcommand\mycommfont[1]{\footnotesize\ttfamily\textcolor{blue}{#1}}
\SetCommentSty{mycommfont}

\begin{algorithm}[!ht]
    \DontPrintSemicolon
	\caption{Training of UPure}
	\label{algo: UPure}
	\textbf{Input}: Training data $D_{train} = D_{\ell} \cup D_{u}$, an SSL learning algorithm $\mathcal{A}_{ss\ell}$, a loss function $\mathcal{L}_{ss\ell}$, a DCT function $\mathcal{T}_{dct}$ and an inverse DCT $\mathcal{T}_{idct}$, three strategies     $\mathcal{S} = \{S_1, S_2, S_3\}$, where $S_1$ is ``Turn to zero,'' $S_2$ is ``Replace from other,'' and $S_3$ is ``Add perturbation,'' and a predefined region $\tau \times \tau$ for perturbation.
 
    \textbf{Output}: Clean Model $\mathcal{M}^*$.

    \tcc{Step 0: Pick a strategy}
    $S^\star \leftarrow \mathcal{S}$
    
    \tcc{$D_u$ is a unlabeled dataset}
    \For{img $\in D_{u}$}{
        $D_u = D_u \backslash \{img\}$ \tcc*{remove image from $D_u$}
        \tcc{Step 1: Transform to DCT spectrum}
        $spectrum = \mathcal{T}_{dct}$($img$)
        
        \tcc{Step 2: Apply UPure}
        \If{$S^*== S_1$}
        {
            $spectrum[-\tau:, -\tau:]$ = $\mathbf{0}$
        }
        \ElseIf{$S^*== S_2$}
        {
            $img_{\ell} \sim D_{\ell}$  \tcc*{randomly select an image from $D_{\ell}$} 
            $spectrum_{\ell}$ = $\mathcal{T}_{dct}$($img_{\ell}$) \\
            $spectrum[-\tau:, -\tau:]$ = $\mathbf{0}$ \\
            $spectrum[-\tau:, -\tau:]$ += $spectrum_{\ell}$
        }
        \ElseIf{$S^*== S_3$}
        {
            $\eta \sim \mathcal{N}(\mathbf{0}, \sigma^2I)$ \tcc*{sample a noise} 
            $spectrum[-\tau:, -\tau:]$ += $\eta$ \\
        }
        \tcc{Step 3: Inverse transform to pixel domain}
        $img = \mathcal{T}_{idct}(spectrum)$ \\
        Apply cutout operation on $img$ \\
        $D_u = D_u \cup \{img\}$ \tcc*{place image back into $D_u$}
    }    \tcc{Adopt an SSL algorithm to train a model $\mathcal{M}^*$}
     Randomly initialize a model $\mathcal{M}$\\
     $\mathcal{M}^*$=$\mathcal{A}_{ss\ell}(D_{train}, \mathcal{L}_{ss\ell}, \mathcal{M})$ 
    \Return Model $\mathcal{M}^*$
\end{algorithm}

%% file: tables/tab-supp-metric.tex
\setlength{\tabcolsep}{4pt}
\begin{table}[!htbp]
\begin{center}
\caption{Quantitative results with our three strategies on CIFAR10.}
\label{tab:supp-metric}
\scalebox{0.9}{
\begin{tabular}{lcc}
\toprule
\multirow{2}{*}{Strategy} & \multicolumn{2}{c}{CIFAR10}\\
\cmidrule{2-3}
 & PSNR & SSIM\\
\midrule
Turn to zero&33.03&$0.9618{\pm.075}$\\
Replace from other &30.65&$0.9461{\pm.088}$\\
Add perturbation &45.43&$0.9969{\pm.012}$\\
\bottomrule
\end{tabular}}
\end{center}
\vspace{-0.15in}
\end{table}

\setlength{\tabcolsep}{1.4pt}

%% file: tables/tab-abl-smooth.tex
\setlength{\tabcolsep}{4pt}
\begin{table}[!ht]
\begin{center}
\caption{Results on filter-based data pre-processing methods.}
\label{tab:abl-smooth}
\scalebox{0.85}{
\begin{tabular}{llllllll}
\toprule
\multirow{2}{*}{Defense} & \multicolumn{3}{c}{CIFAR10} & &\multicolumn{3}{c}{SVHN} \\
\cmidrule{2-4} \cmidrule{6-8}
& \multicolumn{1}{c}{BA} & ASR & BA $\downarrow$ & &\multicolumn{1}{c}{BA} & ASR & BA $\downarrow$\\
\midrule
No defense &92.80&97.54& - & &97.61 & 98.76 & -\\
\midrule
Gaussian Filter &64.79&1.42&28.01&&96.40&\underline{2.37}&1.21\\
Bilateral Filter &\underline{88.16}&99.15&4.64&&\textbf{97.43}&99.84&0.18\\
Median Filter &42.32&\underline{1.12}&50.48&&90.83&96.21&6.78\\
UPure &\textbf{91.05}&\textbf{0.00}&1.75&&\underline{96.49}&\textbf{0.00}&1.12\\
\bottomrule
\end{tabular}}
\end{center}
\vspace{-0.1in}
\end{table}

%% file: math_equations/rdp-tradeoff.tex
\begin{theorem}\cite{zhang2021universal}\label{thm:gaussian_rdp}
For a scalar Gaussian source $X \sim \mathcal{N}(\mu_X,\sigma_X ^2)$, the information rate-distortion-perception function under squared error distortion and squared Wasserstein-2 distance is attained by some $\hat{X}$ jointly Gaussian 
with $X$ and is given by

\begin{align}\label{eq:grdp}
  \resizebox{\linewidth}{!}{$
R(D, P)
=\begin{cases}
    \frac{1}{2} \log \frac{\sigma_{X}^{2} (\sigma_{X}-\sqrt{P})^{2}}{\sigma_{X}^{2} (\sigma_{X}-\sqrt{P})^{2}-(\frac{\sigma_{X}^{2}+(\sigma_{X}-\sqrt{P})^{2}-D}{2})^{2}} 
 &\text{ if } \sqrt{P} \leq \sigma_{X}-\sqrt{|\sigma^2_X-D|}, \\[12pt]
    \max\{\frac{1}{2}\log\frac{\sigma^2_X}{D},0\} &\text{ if } \sqrt{P} > \sigma_{X}-\sqrt{|\sigma^2_X-D|}.
\end{cases}$}
\end{align}
\end{theorem}